\documentclass{cmslatex}
\usepackage[paperwidth=7in, paperheight=10in, margin=.875in]{geometry}
 \usepackage[backref,colorlinks,linkcolor=red,anchorcolor=green,citecolor=blue]{hyperref}
\usepackage{amsfonts,amssymb,bm}
\usepackage{amsmath}
\usepackage{graphicx}
\usepackage{cite}
\usepackage{enumerate}
\usepackage{multirow}
\renewcommand{\theequation}{\arabic{section}.\arabic{equation}}

\DeclareMathOperator{\argmin}{argmin}

\DeclareMathOperator{\poly}{poly}

\newcommand{\RR}{\mathbb{R}}
\newcommand{\EE}{\mathbb{E}}
\newcommand{\PP}{\mathbb{P}}

\renewcommand{\SS}{\mathbb{S}}

\newcommand{\cH}{\mathcal{H}}
\newcommand{\cP}{\mathcal{P}}
\newcommand{\cF}{\mathcal{F}}

\newcommand{\cB}{\mathcal{B}}
\newcommand{\cR}{\mathcal{R}}
\newcommand{\cX}{X}

\newcommand{\bx}{\mathbf{x}}

\newcommand{\bu}{\bm{u}}

\newcommand{\bb}{w}

\allowdisplaybreaks
\begin{document}
 \title{A Priori Estimates  of the Population Risk \\for Two-layer Neural Networks}


\author{
  Weinan E \thanks{Department of Mathematics and Program in Applied and Computational  Mathematics, Princeton University, Princeton NJ 08540, USA; Beijing Institute of Big Data Research, Beijing 100871, China (weinan@math.princeton.edu).}
  \and Chao Ma \thanks{Program in Applied and Computational Mathematics, Princeton University, Princeton NJ 08544, USA (chaom@princeton.edu).}
  \and Lei Wu \thanks{Program in Applied and Computational Mathematics, Princeton University, Princeton NJ 08544, USA (leiwu@princeton.edu)}
}

\pagestyle{myheadings} \markboth{A PRIORI ESTIMATES FOR TWO-LAYER NEURAL NETWORKS}{} 
\maketitle

\begin{center}
{\it In memory of Professor David Shenou Cai}
\end{center}


\begin{abstract}
New estimates for the population risk are established for  two-layer neural networks. These estimates are nearly optimal in the sense
that the error rates scale in the same way as the Monte Carlo error rates. They are  equally effective in the over-parametrized regime when the network size is much larger than the size of the dataset.  These new estimates are a priori in nature in the sense that the bounds depend only on some norms of the underlying functions to be fitted,  not the parameters in the model, in contrast with  most existing results which are a posteriori in nature.  Using these a priori estimates, we provide a perspective for understanding why two-layer neural networks  perform better than {the related} kernel methods.
\end{abstract}



\section{Introduction} 

One of the main challenges in theoretical machine learning is to understand the errors in neural network models \cite{zhang2016understanding}.
To this end, it is useful to draw an analogy with 
classical approximation theory and finite element analysis \cite{ciarlet2002finite}.
There are two kinds of error bounds in finite element analysis depending on whether the target solution (the ground truth) or the numerical
solution enters into the bounds. Let $f^*$ and $\hat{f}_n$ be the true solution and the ``numerical solution'', respectively. ``A priori'' error estimates usually take the form

\[
\|\hat{f}_n - f^*\|_1 \leq C n^{-\alpha}\|f^*\|_2.
\]
where only norms of the true solution enter into the bounds.
In ``a posteriori'' error estimates,  the norms of the numerical solution enter into the bounds:
\[
\|\hat{f}_n - f^*\|_1 \leq C n^{-\beta} \|\hat{f}_n\|_3.
\]
Here $\|\cdot\|_1, \|\cdot\|_2, \|\cdot\|_3$ denote various norms.
In this language,  most recent theoretical results~\cite{neyshabur2015norm,bartlett2017nips,pmlr-v75-golowich18a,neyshabur2017,neyshabur2018a,neyshabur2018towards} on
estimating the generalization error of neural networks should  be viewed as ``a posteriori'' analysis, since the bounds depend on various norms of the neural network model obtained after the training process. 
As was observed in \cite{dziugaite2017computing,arora2018stronger,neyshabur2018towards}, the numerical values of these norms 
are very large, yielding vacuous bounds. For example, \cite{neyshabur2018towards} calculated the  values of various a posteriori bounds for some real two-layer neural networks and it is found that the best bounds are still on the order of $O(10^5)$.



In this paper, we  pursue a different line of attack by providing ``a priori'' analysis. Specifically,  we focus on  two-layer networks, and 
we consider models with  explicit regularization.
 We establish estimates for the population risk  which are  asymptotically sharp with constants depending only on the properties of the target function. 
 Our numerical results  suggest that such regularization terms are necessary in order for the model to be
``well-posed'' (see Section \ref{sec: discuss} for the precise meaning). 

Specifically, our main contributions are:
\begin{itemize}
  \item We establish  a priori estimates of the population risk for learning two-layer neural networks with an explicit regularization. These a priori estimates depend on the Barron norm of the target function. The rates with respect to the number of parameters and number of samples are comparable to the Monte Carlo rate. In addition, our estimates hold for high dimensional and over-parametrized regime.  
  \item We make a comparison 
between the neural network and kernel methods  using these a priori estimates. We show that two-layer neural networks can be understood as 
kernel methods with the kernel adaptively selected from the data. 
This understanding partially explains why neural networks perform better than kernel methods in practice. 
\end{itemize}

The present paper is the first in a series of papers in which we analyze neural network models using a classical numerical
analysis perspective.  Subsequent papers will consider deep neural network models \cite{ma2019priori, ma2019analysis},
the optimization and implicit regularization problem using gradient descent dynamics \cite{leiwu2019two,ma2019analysis}
and the general function spaces and approximation theory  in high dimensions \cite{e2019space}.

\section{Related work}
There are two key problems in learning two-layer neural networks: optimization and generalization. Recent progresses on optimization suggest that over-parametrization is the key factor leading to a nice empirical landscape $\hat{L}_n$~\cite{safran2016quality,freeman2016topology,nguyen2018on}, thus facilitating  convergence towards global minima of $\hat{L}_n$ for gradient-based optimizers~\cite{song2018mean,du2018gradient,chizat2018global}. This leaves the generalization property of learning two-layer neural networks more puzzling, since naive arguments would suggest that more parameters implies worse generalization ability. This contradicts what is observed  in practice. In what follows, we survey previous attempts in analyzing the generalization properties of two-layer neural network models.

\subsection{Explicit regularization}
This line of works studies the generalization property of two-layer neural networks with explicit regularization and our work lies in this category.
Let $n, m$ denote the number of samples and number of parameters, respectively.
For two-layer sigmoidal networks, \cite{barron1994approximation} established a risk bound $O(1/m + md\ln(n)/n)$.  
By considering smoother activation functions, \cite{klusowski2016risk} proved 
another bound $O((\ln d/n)^{1/3})$ for the case when $m \approx \sqrt{n}$. 
Both of these results are proved for a regularized estimator. 
In comparison, the error rate established in this paper,
 $O(1/m+\ln n\sqrt{\ln d/n})$ is sharper and in fact nearly optimal,
and it is also applicable for the over-parametrized regime. For a better comparison, please 
refer to Table \ref{tab: comp}.

\begin{table}[!h]
\renewcommand{\arraystretch}{1.3}
\centering
\begin{tabular}{c|cc}
\hline 
            & rate & over-parametrization \\\hline \hline 
rate of \cite{barron1994approximation} & $\frac{1}{m} + \frac{md\ln(n)}{n}$ & No \\\hline 
rate of \cite{klusowski2016risk}  &  $\left(\frac{\ln d}{n}\right)^{1/3}$ & No \\\hline 
our rate & $ \frac{1}{m}+\ln(n)(\frac{\ln d}{n})^{1/2}$ & Yes\\
\hline 
\end{tabular}
\caption{ Comparison of the theoretical bounds. The second column are the bounds and the third column indicates 
whether  the bounds are relevant in the over-parametrized regime, i.e. $m\geq n$.}
\label{tab: comp}
\end{table}

 More recently, \cite{wei2018margin} considered explicit regularization for 
classification problems. They proved that for the specific cross-entropy loss, 
the regularization path converges to the maximum margin solutions. 
They also proved an a priori  bound on how the network size  affects the margin. 
However, their analysis is restricted to the case where the data is well-separated. 
Our result does not have this restriction.

\subsection{Implicit regularization} 
Another line of works study  how gradient descent (GD) and stochastic gradient descent (SGD) finds the generalizable solutions. \cite{brutzkus2018sgd} proved that SGD learns over-parametrized networks that provably generalize for binary classification problem. {However,  it is not clear how the population risk depends on the number of samples  for their compression-based generalization bound.} Moreover, their proof highly relies on the strong assumption that the data is linearly separable.   
The experiments in \cite{neyshabur2018towards} suggest that increasing the network width can improve the test accuracy of solutions found by SGD. They tried to explain this phenomena by an initialization-dependent (a posterior) generalization bound. However, in their experiments, the largest width $m\approx n$, rather than $m\gg n$. Furthermore their generalization bounds are arbitrarily loose in practice. So their result cannot tell us whether GD can find generalizable solutions for arbitrarily wide networks.  

In \cite{daniely2017sgd} and \cite{allen2018learning}, it is proved that GD with a particularly chosen initialization, learning rate and  early stopping can find generalizable solutions $\theta_T$ such that $L(\theta_T)\leq \min_{\theta} L(\theta) +\varepsilon$, as long as $m\geq \poly(n,\frac{1}{\varepsilon})$. These results differ from ours in several aspects.
 First, both of them  assume that the target function $f^*\in \mathcal{H}_{\pi_0}$, where $\pi_0$ is the uniform distribution over $S^d$. Recall that $\cH_{\pi_0}$ is the reproducing kernel Hilbert space (RKHS) induced by $k_{\pi_0}(x,x')=\mathbb{E}_{w\sim\pi_0}[\sigma(\langle w,x\rangle)\sigma(\langle w,x'\rangle)]$, which is much smaller than $\cB_2(\cX)$, the space we consider. Secondly, through carefully analyzing the polynomial order in two papers, we can see that the sample complexities they provided scales as $O(1/n^{1/4})$, which is worse than $O(1/\sqrt{n})$ proved here. {See also 
\cite{arora2019fine,cao2019generalization} for some even more recent results.}
 
 Recent work in \cite{leiwu2019two,ma2019analysis} has shown clearly that for the kind of initialization schemes considered in
 these previous works or in the over-parametrized regime,  the neural network models do not perform better than the
 corresponding kernel method with a kernel defined by the initialization.  These results do not rule out the possibility that neural network
 models can still outperform kernel methods in some regimes, but they do show that finding these regimes is quite non-trivial.

\section{Preliminaries}
\label{sec: setup}
We begin by recalling the basics of two-layer neural networks and their approximation properties.

The problem of interest is to learn a function from a  training set of $n$ examples 
$S=\{(x_i,y_i)\}_{i=1}^n$, i.i.d. samples drawn from an underlying distribution $\rho_{x,y}$, which is assumed fixed but known only through the samples. Our target function is $f^*(x)=\EE[y|x]$. We assume that the values of $y_i$ are given through the
decomposition $y=f^*(x)+\xi$, where $\xi$ denotes the noise. 
For simplicity, we assume that the data lie in $\cX=[-1,1]^d$ and $0  \le f^* \le 1$.

The two-layer  neural network is defined by
\begin{equation}\label{eqn: 2-layer-net}
    f(x;\theta) = \sum_{k=1}^m a_k \sigma(\bb_k^Tx),
\end{equation}
 where $\bb_k\in\RR^d$,  $\sigma:\RR\mapsto\RR$ is a nonlinear scale-invariant activation function  such as ReLU~\cite{Krizhevsky2012a} and Leaky ReLU~\cite{He2015b}, both satisfies the condition $\sigma(\alpha t)= \alpha \sigma(t)$ for any $\alpha\geq 0,t\in\RR$. Without loss of generality, we assume $\sigma$ is 1-Lipschitz continuous. In the formula~\eqref{eqn: 2-layer-net}, we omit the bias term for notational simplicity.  The effect of bias term can be incorporated if we assume that the first component of $x$ is always 1.  We say that a network is over-parametrized if  the \textit{network width} $m>n$.  We define a truncated form of $f$ through $Tf(x)=\max\{\min\{f(x),1\},0\}$. By an abuse of notation, in the following we still use $f$ to denote $Tf$.
 We will use $\theta = \{(a_k,\bb_k)\}_{k=1}^m$  to denote  all the parameters  to be learned from the training data,

 The ultimate goal  is to minimize the population risk
\[
    L(\theta) = \EE_{x,y}[\ell(f(x;\theta),y)].
\]
In practice, we have to work with the empirical risk
\[
    \hat{L}_n(\theta) = \frac{1}{n}\sum_{i=1}^n \ell(f(x_i;\theta),y_i).
\]
Here the loss function $\ell(y,y')=\frac{1}{2}(y-y')^2$, unless it is specified otherwise.

Define the path norm~\cite{neyshabur2015norm},
\begin{equation}\label{eqn: def-path-norm}
\|\theta\|_{\cP}:=\sum_{k=1}^m |a_k|\|w_k\|_1,
\end{equation}
We will consider the regularized model defined as follows:
\begin{definition}\label{def: estimator}
For a two-layer neural network $f(\cdot;\theta)$ of width $m$, we define the regularized risk as 
\begin{equation*}\label{eqn: prr-noise}
J_{\lambda}(\theta) := \hat{L}_n(\theta) +  \lambda (\|\theta\|_{\cP}+1).
\end{equation*}
The $+1$ term at the right hand side is included only to simplify the proof. Our result also holds if we do not include
this term in the regularized risk.
The corresponding regularized estimator is defined as 
\begin{equation*}\label{eqn: minimizer}
    \hat{\theta}_{n,\lambda} =  \argmin J_{\lambda}(\theta).
\end{equation*}
\end{definition}
Here $\lambda>0$ is a tuning parameter that controls the balance between the fitting error and the model complexity. It is worth noting that the minimizer is not necessarily unique, and $\hat{\theta}_{n,\lambda}$ should be understood as any of the minimizers. 

In the following, we will call Lipschitz continuous functions with Lipschitz constant $C$ $C$-Lipschitz continuous. We will use $X\lesssim Y$ to indicate that $X\leq c Y$ for some universal constant $c>0$.

\subsection{Barron space}
\label{sec: barron-space}

We begin by defining the natural function space associated with two-layer neural networks, which we will refer to 
as the Barron space to honor the 
{
 pioneering work that Barron has done on this subject ~\cite{barron1993universal,klusowski2017minimax,klusowski2016risk,klusowski2018approximation}. 
 A more complete discussion can be found in \cite{e2019space}.
 
Let $\SS^d:= \{w \,|\, \|w\|_1=1\}$, and let $\mathcal{F}$ be the Borel $\sigma$-algebra on $\SS^d$ and $P(\SS^d)$ be the collection of all probability measures on $(\SS^d, \mathcal{F})$. 
Let $\cB(\cX)$ be the collection of functions that admit the following integral representation:
\begin{equation}\label{eqn: integral-rep}
    f(x) = \int_{\SS^d} a(w)\sigma(\langle w,x\rangle) d\pi(w) \quad \forall x \in \cX,
\end{equation}
where $\pi\in P(\SS^d)$, and $a(\cdot)$ is a measurable function with respect to $(\SS^d,\cF)$. 
 For any $f\in \cB(\cX)$ and $p\geq 1$, we define the following norm
\begin{equation}\label{eqn: barron-norm}
\gamma_{p}(f) := \inf_{(a,\pi)\in \Theta_f}\left(\int_{\SS^d} |a(w)|^p d\pi(w)\right)^{1/p},
\end{equation}
where 
\[
\Theta_f=\big\{(a,\pi) \,|\, f(x)=\int_{\SS^d} a(w)\sigma(\langle w,x\rangle) d\pi(w)\big\}.
\] 

\begin{definition}[Barron space]
We define Barron space by
\[
\cB_p(\cX):=\{ f\in \cB(\cX)\ |\ \gamma_p(f)< \infty\}.
\]
\end{definition}

Since $\pi(\cdot)$ is a probability distribution, by H\"{o}lder's inequality, for any $q\geq p>0$ we have 
$
    \gamma_p(f) \leq \gamma_q(f).
$
Thus, we have $\mathcal{B}_{\infty}(\cX) \subset \cdots \subset \mathcal{B}_2(\cX)\subset \mathcal{B}_{1}(\cX)$.  
}

Obviously $\mathcal{B}_p(\cX)$ is dense in $C(\cX)$ since  all the finite two-layer neural networks belong to Barron space with $\pi(w)=\frac{1}{m}\sum_{k=1}^m\delta(w-\hat{w}_k)$ and 
the universal approximation theorem~\cite{cybenko1989approximation} tells us that continuous functions can be approximated by two-layer neural networks. 
Moreover, it is interesting to note that the $\gamma_1(\cdot)$ norm of a two-layer neural network is bounded by the path norm of the parameters.

An important result proved in  { \cite{breiman1993hinging,klusowski2016risk} states that if a function $f: \cX\mapsto \RR$ satisfies  $\int_{\RR^d} \|\omega\|_1^{2} |\hat{f}(\omega)| d\omega <\infty$, where $\hat{f}$ is the Fourier transform of an extension of $f$, then it can be expressed in the form~\eqref{eqn: integral-rep} with  
\[
\gamma_{\infty}(f):=\sup_{w\in\SS^d} |a(w)| \lesssim \int_{\RR^d} \|\omega\|_1^{2} |\hat{f}(\omega)| d\omega.
\] 
Thus it lies in $\cB_{\infty}(\cX)$. } 

\paragraph*{\bf Connection with reproducing kernel Hilbert space}
The Barron space has a natural connection with reproducing kernel Hilbert space (RKHS)~\cite{aronszajn1950theory}, and as we will show later, this connection will lead to a precise comparison between two-layer neural networks and kernel methods.
For a fixed $\pi$, 
we define
\[
    \mathcal{\cH}_{\pi}(\cX) := \left\{ 
    \int_{\SS^d} \alpha(w)\sigma(\langle w,x \rangle)d\pi(w)\, : \|f\|_{\cH_{\pi}} <\infty
    \right\},
\]
where 
\[
\|f\|^2_{\cH_{\pi}} := \EE_{\pi}[|a(w)|^2 ].
\]
Recall that for a symmetric positive definite (PD)\footnote{We say $k$ is PD function, if for any $x_1,\dots,x_n$, the matrix $ K^n$ with $K^n_{i,j}=k(x_i,x_j)$ is positive semidefinite.} function $k:\cX \times \cX \mapsto \RR$, the induced RKHS $\cH_k$ is the completion of $\{\sum_i a_i k(x_i,x)\}$ with respect to  the inner product $\langle k(x_i,\cdot), k(x_j,\cdot)\rangle_{\cH_k}=k(x_i,x_j)$.
It was proved in \cite{rahimi2008uniform}  that $\cH_{\pi}=\cH_{k_{\pi}}$ with the kernel $k_{\pi}$ defined by 
\begin{equation}\label{eqn: kernel}
    k_{\pi}(x,x') = \EE_{\pi}[ \sigma(\langle w,x\rangle)\sigma(\langle w,x'\rangle)].
\end{equation}
Thus Barron space can be viewed as the union of a family of RKHS  with  kernels defined by $\pi$ through Equation~\eqref{eqn: kernel}, i.e. 
\begin{equation}\label{eqn: barron-rkhs}
\mathcal{B}_2(\cX) = \bigcup_{\pi\in P(\SS^{d})} \cH_{\pi}(\cX).
\end{equation}
Note that the family of  kernels is only determined by the activation function $\sigma(\cdot)$.

\subsection{Approximation property}
 \begin{theorem}\label{pro: approximation}
 For any $f\in\mathcal{B}_2(\cX)$, there exists a two-layer neural network $f(\cdot;\tilde{\theta})$ of width $m$, such that
 \begin{align}
    \mathbb{E}_{x}[(f(x)-f(x;\tilde{\theta}))^2] & \leq \frac{3 \gamma_2^2(f)}{m} \label{eqn: approx-rate} \\
    \|\tilde{\theta}\|_{\cP} &\leq 2\gamma_2(f) \label{eqn: explicit-norm}
 \end{align}
 \end{theorem}
{ This kind of approximation results have been established in many papers, see for example~\cite{barron1993universal,breiman1993hinging}. The difference is that we provide the explicit control of the norm of the constructed solution in \eqref{eqn: explicit-norm},
and the bound is independent  of the network size. This observation will be useful for what follows.

 The proof of Proposition~\ref{pro: approximation} can be found in  Appendix A.
 The basic intuition is that the integral representation of $f$ allows us to approximate $f$ by the Monte-Carlo method:
 $f(x)\approx \frac{1}{m}\sum_{k=1}^m a(w_k) \sigma(\langle w_k,x\rangle)$ where $\{w_k\}_{k=1}^m$ are sampled from the distribution $\pi$.  
 }

\section{Main results}
For simplicity we first discuss the case without noise, i.e. $\xi=0$. In the next section, we deal with the noise. 
We also  assume $\ln(2d)\geq 1$, and let $\hat{\gamma}_p(f)=\max\{1,\gamma_p(f)\}, \lambda_n = 4\sqrt{2\ln(2d)/n}$. {Here $d$ is the dimension of input  and the definition of $\gamma_p(\cdot)$ is given in Equation \eqref{eqn: barron-norm}.}

\begin{theorem}[Noiseless case]\label{thm: priori-err-1}
Assume that  the target function $f^*\in \cB_2(\cX)$ and $\lambda \geq \lambda_n$. Then for any $\delta>0$, with probability  at least $1-\delta$ over the choice of the training set $S$, we have 
\begin{align}\label{eqn: priori-bound}
\EE_{x}|f(x;\hat{\theta}_{n,\lambda})-&f^*(x)|^2 \lesssim   \frac{\gamma^2_2(f^*)}{m} + \lambda \hat{\gamma}_2 (f^*) \\
& + \frac{1}{\sqrt{n}}\big(\hat{\gamma}_2(f^*)+\sqrt{\ln(n/\delta)}\big) .
\end{align}
\end{theorem}
The above theorem provides an a priori estimate for the population risk. The a priori nature  is reflected by dependence of the $\gamma_2(\cdot)$ norm of the target function.
The first term at the right hand side controls the approximation error. The second term
 bounds the estimation error.  Surprisingly, the bound for the estimation error  is independent of the network width $m$.
 Hence the bound also makes sense in  the over-parametrization regime. 

In particular, if we take $\lambda\asymp \lambda_n$ and $m\geq \sqrt{n}$, the bound becomes $O(1/\sqrt{n})$ up to some logarithmic terms. This bound is nearly optimal in a minimax sense~\cite{yang1999information,klusowski2017minimax}.

\subsection{Comparison with kernel methods}
Consider $f^*\in \cB_2(\cX)$, and without loss of generality, we assume that $(a^*,\pi^*)\in \Theta_{f^*}$ is one of the best representations of $f^*$ (it is easy to prove that such a representation exists), i.e.
$
    \gamma_2^2(f^*)= \EE_{\pi^*}[|a^*(w)|^2].
$
For a fixed $\pi_0$, we have,
\begin{equation}
\begin{aligned}
f^*(x) &= \int_{\SS^d} a^*(w) \sigma(\langle w,x\rangle) d\pi^*(w)\\
&=\int_{\SS^d}a^*(w) \frac{d\pi^*}{d\pi_0}(w) \sigma(\langle w,x\rangle) d\pi_0(w)
\end{aligned}
\end{equation}
as long as $\pi$ is absolutely continuous with respect to $\pi_0$.
In this sense, we can view $f^*$ from the perspective of   $\cH_{\pi_0}$. 
Note that $\cH_{\pi_0}$ is induced by PD function $k_{\pi_0}(x,x')=\EE_{w\sim \pi_0}[\sigma(\langle w,x\rangle)\sigma(\langle w,x'\rangle)]$, and the norm  of $f^*$ in $\cH_{\pi_0}$ is given by 
\[
    \|f^*\|^2_{\cH_{\pi_0}} = \EE_{w\sim\pi_0}[|a^*(w)\frac{d\pi^*}{d\pi_0}(w)|^2].
\]
Let $\hat{h}_{n,\lambda}$ be the solution of the kernel ridge regression (KRR) problem defined by:
\begin{align}
     \min_{h\in\cH_{\pi_0}} \frac{1}{2n}\sum_{i=1}^n (h(x_i)-y_i)^2 + \lambda \|h\|_{\cH_{\pi_0}}.
\end{align}
We are interested in the comparison between the two population risks $L(\hat{\theta}_{n,\lambda})$ and $L(\hat{h}_{n,\lambda})= \EE[\ell(\hat{h}_{n,\lambda}(x),y)]$.

If $\|f^*\|_{\cH_{\pi_0}}<\infty$, then we have   $ f^*\in \cH_{\pi_0}$ and $\inf_{h\in\cH_{\pi_0}} L(h) =0$. In this case,  
it was proved in  \cite{caponnetto2007optimal} that  the optimal learning rate is 
\begin{equation}\label{eqn: 999}
L(\hat{h}_{n,\lambda})  \sim  \frac{\|f^*\|_{\cH_{\pi_0}}}{\sqrt{n}}.
\end{equation}
Compared to Theorem~\ref{thm: priori-err-1}, we can see that both rates  have the same scaling with respect to  $n$, the number of samples. The only  difference appears in the two norms: $\gamma_2(f^*)$ and $\|f^*\|_{\cH_{\pi_0}}$. From the definition~\eqref{eqn: barron-norm}, we always have $\gamma_2(f^*)\leq \|f^*\|_{\cH_{\pi_0}}$,  since   $(a^* \frac{d\pi^*}{d\pi_0},\pi_0)\in \Theta_{f^*}$. 
If $\pi^*$ is nearly singular with respect to $\pi_0$, then  $\|f^*\|_{\cH_{\pi_0}}\gg \gamma_2(f^*)$. In this case, the population
risk for the kernel methods should be much larger than the population risk for the neural network model.

\paragraph*{\bf Example }
Take $\pi_0$ to be the uniform distribution over $\SS^d$ and $f^*(x)=\sigma(\langle w^*,x\rangle)$, for which $\pi^*(w)=\delta(w-w^*)$ and $a^*(w)=1$.  In this case $\gamma_2(f^*)=1$, but $\|f^*\|_{\cH_{\pi_0}}=+\infty$. Thus the rate~\eqref{eqn: 999} becomes trivial. 
Assume that  the {population risk} scales as $O(n^{-\beta})$, and it is interesting to see how $\beta$ depends on the dimension $d$. 
We numerically estimate $\beta$'s for two methods, and report the results in Table \ref{tab: one-neuron}.
It does show that the higher the dimensionality, the slower the rate of the kernel method. In contrast, the rates for the two-layer neural networks are independent of the dimensionality, which confirms the the prediction of Theorem \ref{thm: priori-err-1}.  {For this particular target function,
the value of $\beta\geq 1$ is bigger than the lower bound ($1/2$) proved in
Theorem \ref{thm: priori-err-1}. This is not a contradiction since the latter holds for 
any $f\in \cB_2(\cX)$.}
\begin{table}[!h]
\renewcommand{\arraystretch}{1.1}
\centering
\begin{tabular}{|c|ccc|}
\hline 
$d$  & $10$ & $100$ & $1000$ \\\hline 
$\beta_{\text{nn}}$ & $1.18$ & $1.23$ & $1.02$ \\
$\beta_{\ker}$ & $0.50$ & $0.35$ & $0.14$ \\
\hline
\end{tabular}
\caption{The error rates of learning the one-neuron function in different dimensions. 
The second and third lines correspond to the two-layer neural network and the kernel ridge regression method, respectively.}
\label{tab: one-neuron}
\end{table}


\paragraph*{\bf The two-layer neural network model as of an adaptive kernel method}
Recall that $\cB_2(\cX)=\cup_{\pi} \cH_{\pi}(\cX)$. The norm $\gamma_2(\cdot)$ characterizes the complexity of the
target function by selecting the best kernel among a family of kernels $\{k_{\pi}(\cdot,\cdot)\}_{\pi\in P(\SS^d)}$.  The kernel method works with a specific RKHS with a particular choice of the kernel or the probability distribution $\pi$. In contrast, the neural network models work with the union of all these
RKHS and select the kernel or the probability distribution adapted to the data.
{From this perspective, we can view  the  two-layer neural network model as an adaptive kernel method.}

\subsection{Tackling the noise}
We first make the following sub-Gaussian assumption on the noise.
\begin{assumption}\label{assump: noise}
We assume that the noise satisfies
\begin{equation}
    \PP[|\xi|> t] \leq c_0 e^{-\frac{t^2}{\sigma}} \,\,\, \ \forall\, t\geq \tau_{0}.
\end{equation}
Here $c_0,\tau_0$ and $\sigma$ are constants. 
\end{assumption}


In the presence of noise, the population risk can be decomposed into 
\begin{equation}\label{eqn: decomposition}
L(\theta) = \EE_{x} (f(x;\theta)-f^*(x))^2  + \EE[\xi^2].
\end{equation}
This suggests that, in spite of the noise, we still have
$
    \argmin_{\theta} L(\theta)  = \argmin_{\theta} \EE_{x} | f(x;\theta)-f^*(x)|^2,
$
and the latter is what we really want to minimize. However due to the noise, $\ell(f(x_i),y_i)$ might be unbounded. We  cannot directly use the generalization bound in Theorem~\ref{thm: posterior-gen-gap}. To address this issue, 
we  consider the truncated risk defined as follows,
\begin{align*}
L_B(\theta)&=\EE_{x,y}[\ell(f(x;\theta),y)\wedge \frac{B^2}2]\\ 
\hat{L}_{B}(\theta) &= \frac{1}{n} \sum_{i=1}^n \ell(x_i;\theta),y_i)\wedge \frac{B^2}2.
\end{align*}
Let $B_n=1 + \max\{\tau_0,\sigma^2\ln n\}$. For the noisy case, we consider the following regularized risk: 
\begin{equation}
J_{\lambda}(\theta) := \hat{L}_{B_n}(\theta) +  \lambda B_n (\|\theta\|_{\cP}+1).
\end{equation}
The corresponding regularized estimator is given by 
$
    \hat{\theta}_{n,\lambda} = \argmin J_{\lambda}(\theta).
$
Here for simplicity we slightly abused the notation.

\begin{theorem}[Main result, noisy case]\label{thm: priori-err-2}
Assume that the target function $f^*\in \cB_2(\cX)$ and $\lambda \geq \lambda_n$. Then for any $\delta>0$, with probability  at least $1-\delta$ over the choice of the training set $S$, we have 
\begin{align}\nonumber
\EE_{x}|f(x;\hat{\theta}_{n,\lambda})&-f^*(x)|^2 \lesssim   \frac{\gamma^2_2(f^*)}{m} + \lambda B_n\hat{\gamma}_2 (f^*) \\\nonumber
& + \frac{B_n^2}{\sqrt{n}} \Big(\hat{\gamma}_2(f^*)+\sqrt{\ln(n/\delta)}\Big) \\\nonumber
& + \frac{B_n^2}{\sqrt{n}}\big(c_0\sigma^2 + \sqrt{\frac{\EE[\xi^2]}{n^{1/2}\lambda}}\big).
\end{align}
\end{theorem}
Compared to Theorem~\ref{thm: priori-err-1},  the noise introduces at most several logarithmic terms. The case with no noise
corresponds to the situation with  $B_n=1$.

\subsection{Extension to classification problems}
Let us consider the simplest setting: binary classification problem, where $y\in\{0,1\}$. In this case, $f^*(x)=\PP\{y=1|x\}$ denotes the probability of $y=1$ given $x$. Given  $f^*(\cdot)$ and $f(\cdot;\theta_{n,\lambda})$, the corresponding plug-in classifiers are defined by $\eta^*(x)=1[f^*(x)\geq \frac{1}{2}]$ and $\hat{\eta}(x)=1[f(x;\hat{\theta}_{n,\lambda})\geq \frac{1}{2}]$, respectively. { $\eta^*$ is the optimal Bayes classifier.

For a classifier $\eta$, we measure its performance by the 0-1 loss defined by $\mathcal{E}(\eta) = \PP\{\eta(x)\neq y\}$. 
\begin{corollary}
Under the same assumption as in Theorem~\ref{thm: priori-err-2} and taking $\lambda =\lambda_n$, for any
$\delta \in (0, 1)$, with probability at least $1-\delta$, we have 
\begin{align*}
\mathcal{E}(\hat{\eta})&\lesssim \mathcal{E}(\eta^*) + \frac{\gamma_2(f^*)}{\sqrt{m}} + \hat{\gamma}_2^{1/2}(f^*)\frac{\ln^{1/4} (d) + \ln^{1/4}(n/\delta)}{n^{1/4}} .
\end{align*}
\end{corollary}
\begin{proof}
According to the Theorem 2.2. of \cite{devroye2013probabilistic}, we have 
\begin{align}\label{eqn: 0000}
\mathcal{E}(\hat{\eta}) - \mathcal{E}(\eta^*) &\leq 2 \EE[|f(x;\hat{\theta}_{n,\lambda})-f^*(x)|]\\
\nonumber&\leq 2 \EE[|f(x;\hat{\theta}_{n,\lambda})-f^*(x)|^2]
\end{align}
In this case, $\varepsilon_i=y_i-f^*(x_i)$ is bounded by $1$, thus $\tau_0=1,c=\sigma=0$. 
Applying Theorem~\ref{thm: priori-err-2} yields the result. 
\end{proof}

The above theorem suggests that our a priori estimates also hold for classification problems, although the error rate only scales as $O(n^{-1/4})$. It is possible to improve the rate with more a delicate analyses. One potential way is to specifically develop a better estimate for $L_1$ loss, as can be seen from inequality~\eqref{eqn: 0000}. Another way is to make a stronger assumption on the data. For example, we can assume that there exists $f^*\in\cB_2(\cX)$ such that $\PP_{x,y}(yf^*(x)\geq 1)=1$, for which the Bayes error $\mathcal{E}(\eta^*)=0$. 
We leave these to future work.
}

\section{Proofs}

\subsection{Bounding the generalization gap}
\begin{definition}[Rademacher complexity]
Let $\cF$ be a hypothesis space, i.e. a set of functions. The Rademacher complexity of $\cF$ with respect to samples $S=(z_1,z_2,\dots,z_n)$ is defined as 
$
\hat{\cR}_n(\cF) = \frac{1}{n}\EE_{\bm{\varepsilon}}[\sup_{f\in \cF} \sum_{i=1}^n\varepsilon_i f(z_i)],
$
where $\{\varepsilon_i\}_{i=1}^n$ are i.i.d. random variables with   $\PP(\varepsilon_i=+1)=\PP(\varepsilon_i=-1)=\frac{1}{2}$. 
\end{definition}
 The generalization gap can be estimated via the Rademacher complexity by the following theorem~\cite{shalev2014understanding} .
\begin{theorem}\label{thm: gen-err-rademacher-complexity}
Fix a hypothesis space $\cF$. Assume that  for any $f\in \cF$ and $z$, $|f(z)|\leq B$. Then~for any $\delta>0$,  with probability at least $1-\delta$ over the choice of $S=(z_1,z_2,\dots,z_n)$, we have,
\[
    |\frac{1}{n}\sum_{i=1}^n f(z_i) - \EE_{z}[f(z)]| \leq 2 \EE_S[\hat{\cR}_n(\cF)] + B\sqrt{\frac{2\ln(2/\delta)}{n}}.
\]
\end{theorem}
Let $\cF_Q = \{ f(x;\theta)\,|\, \|\theta\|_{\cP}\leq Q\}$ denote all the two-layer networks with path norm bounded by $Q$. It was proved in \cite{neyshabur2015norm} that 
\begin{equation}
    \hat{\cR}_n(\cF_Q) \leq 2Q \sqrt{\frac{2\ln(2d)}{n}}.
\end{equation}
By combining the above result withTheorem~\ref{thm: gen-err-rademacher-complexity}, we  obtain the following a posterior bound of the generalization gap for two-layer neural networks. The proof is deferred to Appendix B.

\begin{theorem}[A posterior generalization bound]\label{thm: posterior-gen-gap}
Assume that the loss function $\ell(\cdot,y)$ is $A-$Lipschitz continuous and bounded by $B$. Then
for any  $\delta > 0$, with probability  at least $1-\delta$ over the choice of the training set $S$, we have, for any two-layer network $f(\cdot;\theta)$, 
 \begin{align}\label{eqn: pgb}
    |L(\theta)-\hat{L}_n(\theta)| \leq &\,\,  4A \sqrt{\frac{2\ln(2 d)}{n}} \left(\|\theta\|_{\cP}+1\right)\\
    &\quad + B\sqrt{\frac{2\ln(2c (\|\theta\|_{\cP}+1)^2/\delta)}{n}},
\end{align}
where $c=\sum_{k=1}^{\infty}1/k^2$.
\end{theorem}

We  see that the generalization gap is  bounded roughly by $\|\theta\|_{\cP}/\sqrt{n}$ up to some logarithmic terms.

\subsection{Proof for the noiseless case}
\label{sec: priori-analysis}
The intuition is as follows. The path norm of the special solution $\tilde{\theta}$ which achieves the optimal approximation error is independent of the network width, and this norm can also be used to bound the generalization gap (Theorem~\ref{thm: posterior-gen-gap}). Therefore, if the path norm is suitably penalized during training, we should be able to control the generalization gap without harming the approximation accuracy.

We first have the estimate for the regularized risk of $\tilde{\theta}$. 
\begin{proposition}\label{pro: contructed-sol}
Let $\tilde{\theta}$ be the network constructed in Theorem~\ref{pro: approximation}, and $\lambda\geq \lambda_n$.
 Then  with probability at least $1-\delta$, we have
\begin{equation}\label{eqn: special-solution}
J_{\lambda}(\tilde{\theta}) \leq L(\tilde{\theta})  + 8\lambda \hat{\gamma}_2(f^*) + 2\sqrt{\frac{2\ln(2c/\delta)}{n}}
\end{equation}
\end{proposition}
\begin{proof}
First $\ell(y,y_i)=\frac{1}{2}(y-y_i)^2$ is $1$-Lipschitz continuous and bounded by $2$.
According to Definition~\ref{def: estimator} and the property that $ \|\tilde{\theta}\|_{\cP} \leq 2 \gamma_2(f^*)$,  the regularized risk of $\tilde{\theta}$ satisfies
\begin{align}\label{eqn: 111}\nonumber
J_{\lambda}(\tilde{\theta}) & \stackrel{}{=}  \hat{L}_n(\tilde{\theta}) + \lambda  (\|\tilde{\theta}\|_{\cP} +1)\\ \nonumber 
    &\leq L(\tilde{\theta}) + (\lambda_n+\lambda) (\|\tilde{\theta}\|_{\cP}+1)  + 2\sqrt{\frac{2\ln(2c (\|\tilde{\theta}\|_{\cP}+1)^2/\delta)}{n}} \\
    &\stackrel{}{\leq} L(\tilde{\theta}) +  6\lambda \hat{\gamma}_2(f^*) + 2\sqrt{\frac{2\ln(2c(1+2\gamma_2(f^*))^2/\delta)}{n}}.
\end{align}
 The last term can be simplified by using $\sqrt{a+b}\leq \sqrt{a}+\sqrt{b}$ and $\ln(1+a)\leq a$ for $a\geq 0, b\geq 0$. So we have 
\begin{eqnarray*}
\sqrt{2\ln(2c(1+2\gamma_2(f^*))^2/\delta)} &\leq& \sqrt{2\ln(2c/\delta)}+3\hat{\gamma}_2(f^*).
\end{eqnarray*}
Plugging it into Equation~\eqref{eqn: 111} completes the proof.
\end{proof}

\begin{proposition}[Properties of regularized solutions]\label{pro: reg-estimator}
The regularized estimator $\hat{\theta}_{n,\lambda}$ satisfies:
\begin{align*}
    J_{\lambda}(\hat{\theta}_{n,\lambda}) & \leq J_{\lambda}(\tilde{\theta}) \\
    \|\hat{\theta}_{n,\lambda}\|_{\cP} & \leq \frac{L(\tilde{\theta})}{\lambda} +8 \hat{\gamma}_2(f^*)+\frac{1}{2}\sqrt{\ln(2c/\delta)}
\end{align*}
\end{proposition}
\begin{proof}
The first claim follows from the definition of $\hat{\theta}_n$.
For the second claim, note that
 \[
 \lambda (\|\hat{\theta}_{n,\lambda}\|_{\cP} +1)\leq J_{\lambda}(\hat{\theta}_n) \leq J_{\lambda}(\tilde{\theta}),
 \]
  Applying Proposition~\ref{pro: contructed-sol} completes the proof.
 \end{proof}
\vspace*{1mm}
\begin{remark}
The above proposition establishes the connection between the regularized solution and the special solution $\tilde{\theta}$ constructed in Proposition~\ref{pro: approximation}. In particular, by taking $\lambda=t\lambda_n$ with $t\geq 1$  the generalization gap of the regularized solution is bounded by   $\frac{\|\hat{\theta}_{n,\lambda}\|_{\cP}}{\sqrt{n}} \to L(\tilde{\theta})/(t\sqrt{\ln 2d})$ as $n\to\infty$, up to some constant. This suggests that our regularization term is  appropriate, and it forces the generalization gap to be  roughly in the same order as the approximation error. 
\end{remark}

\begin{proof}({\bf Proof of Theorem~\ref{thm: priori-err-1}}) Now we are ready to prove the main result.
 Following the a posteriori generalization bound given in Theorem~\ref{thm: posterior-gen-gap}, we have with probability at least $1-\delta$,
\begin{align*}
L(\hat{\theta}_{n,\lambda}) 
&\stackrel{}{\leq} \hat{L}_n(\hat{\theta}_{n,\lambda})+\lambda_n(\|\hat{\theta}_{n,\lambda}\|_{\cP}+1)  +3 Q_n \\
&\stackrel{(1)}{\leq} J_{\lambda}(\hat{\theta}_{n,\lambda}) +  3Q_n,
\end{align*}
where $Q_n=\sqrt{\ln(2c(1+\|\hat{\theta}_{n,\lambda}\|)^2/\delta)/n}$. The inequality (1) is due to the choice $\lambda\geq \lambda_n$.  The first term can be bounded by $J_{\lambda}(\hat{\theta}_{n,\lambda})\leq J_{\lambda}(\tilde{\theta})$, which is given by Proposition~\ref{pro: contructed-sol}. It remains  to bound $Q_n$,
\begin{align*}
\sqrt{n}Q_n&\leq \sqrt{\ln(2nc/\delta)} + \sqrt{2\ln(1+n^{-1/2}\|\hat{\theta}_{n,\lambda}\|_{\cP})}\\
&\leq  \sqrt{\ln(2nc/\delta)} + \sqrt{2\|\hat{\theta}_{n,\lambda}\|_{\cP}/\sqrt{n}}.
\end{align*}
By Proposition~\ref{pro: reg-estimator}, we have 
\begin{align*}
    \sqrt{\frac{2\|\hat{\theta}_{n,\lambda}\|_{\cP}}{\sqrt{n}}} &\leq \sqrt{\frac{2(L(\tilde{\theta})/\lambda +8 \hat{\gamma}_2(f^*)+0.5\sqrt{\ln(2c/\delta)})}{\sqrt{n}}}\\
    &\leq \sqrt{\frac{2L(\tilde{\theta})}{\lambda n^{1/2}}} + \frac{3\hat{\gamma}_2(f^*)}{n^{1/4}} + \left(\frac{\ln(1/\delta)}{n}\right)^{1/4}.
\end{align*}
Thus after some simplification, we  obtain
\begin{align}\label{eqn: mm}
Q_n\leq 2\sqrt{\frac{\ln(n/\delta)}{n}} + \sqrt{\frac{2L(\tilde{\theta})}{\lambda n^{3/2}}} + \frac{3\hat{\gamma}_2(f^*)}{\sqrt{n}} .
\end{align}

By combining Equation~\eqref{eqn: special-solution} and ~\eqref{eqn: mm}, we obtain
\begin{align*}
    L(\hat{\theta}_n) &\lesssim L(\tilde{\theta})  + 8\lambda \hat{\gamma}_2(f^*) + \frac{3}{\sqrt{n}}\Big(\sqrt{\frac{L(\tilde{\theta})}{n^{1/2}\lambda}} + \hat{\gamma}_2(f^*) + \sqrt{\ln(n/\delta)} \Big).
\end{align*}
By applying  $L(\tilde{\theta})\leq 3\gamma_2^2(f^*)/m$, we complete the proof.
\end{proof}

\subsection{Proof for the noisy case}
We need the following lemma.  The proof  is deferred to Appendix D.
\begin{lemma}\label{lemma: noise}
Under Assumption~\ref{assump: noise}, we have
$$
    \sup_{\theta} |L(\theta) - L_{B_n}(\theta)| \leq \frac{2c_0\sigma^2}{\sqrt{n}},
$$
\end{lemma}

Therefore we have, 
$$L(\theta)= L(\theta)-L_{B_n}(\theta) + L_{B_n}(\theta)\leq \frac{2c_0\sigma^2}{\sqrt{n}} + L_{B_n}(\theta)$$
This suggests that as long as we can bound the truncated population risk,  the original risk  will be bounded accordingly. 

\begin{proof}({\bf Proof of Theorem~\ref{thm: priori-err-2}})
The proof is almost the same as the noiseless case.
The loss function $\ell(y,y_i)\wedge B^2/2$ is $B$-Lipschitz continuous and bounded by $B^2/2$. 
By analogy with the proof of Proposition~\ref{pro: contructed-sol}, we  obtain that with probability at least $1-\delta$ the following inequality holds,
\begin{align}\label{eqn: 333}
J_{\lambda}(\tilde{\theta}) \leq L_{B_n}(\tilde{\theta}) + 8 B_n \lambda \hat{\gamma}_2(f^*) + B_n^2\sqrt{\frac{\ln(2c/\delta)}{n}}.
\end{align} 
Following the proof in Proposition~\ref{pro: reg-estimator},  we similarly obtain $J_{\lambda}(\hat{\theta}_{n,\lambda})  \leq J_{\lambda}(\tilde{\theta})$ and
\begin{align}\label{eqn: 222}
   \|\hat{\theta}_{n,\lambda}\|_{\cP} & \leq \frac{L_{B_n}(\tilde{\theta})}{B_n\lambda} + 8\hat{\gamma}(f^*) + \frac{B_n}{2}\sqrt{\ln(2c/\delta)}.
\end{align}
Following the proof of Theorem~\ref{thm: priori-err-1}, we have 
\begin{align}\label{eqn: aaa}
    L_{B_n}(\hat{\theta}_{n,\lambda}) &\leq J_{\lambda}(\tilde{\theta}) + \frac{B_n^2}{2}\sqrt{2\ln(2c(1+\|\hat{\theta}_{n,\lambda}\|_{\cP})^2/\delta)/n}
\end{align}
Plugging \eqref{eqn: 333} and \eqref{eqn: 222} into \eqref{eqn: aaa},  we get
\begin{align*}
 L_{B_n}(\hat{\theta}_{n,\lambda}) &\leq L_{B_n}(\tilde{\theta}) + 8B_n \hat{\gamma}_2(f^*)\lambda \\
 &\quad + \frac{3B_n^2}{\sqrt{n}} \Big(\sqrt{\frac{L_{B_n}(\tilde{\theta})}{n^{1/2}\lambda}} + \hat{\gamma}_2(f^*)+\sqrt{\ln(n/\delta)}\Big)
\end{align*}
Using  Lemma~\ref{lemma: noise} and the decomposition~\eqref{eqn: decomposition}, we complete the proof.
\end{proof}

\section{Numerical Experiments}
\label{sec: experiments}
In this section, we evaluate  the regularized model using numerical experiments.
We consider two datasets, MNIST\footnote{\url{http://yann.lecun.com/exdb/mnist/}} and CIFAR-10\footnote{\url{https://www.cs.toronto.edu/~kriz/cifar.html}}. Each example in MNIST is a $28\times 28$ grayscale image, while each example in CIFAR-10 is a $32\times 32\times 3$ color image.
For MNIST, we map numbers $\{0,1,2,3,4\}$ to label $0$ and $\{5,6,7,8,9\}$ to $1$. For CIFAR-10, we select the examples with labels $0$ and $1$ to construct our new training and validation sets. Thus, our new MNIST has $60,000$ training examples, and CIFAR-10 has $10,000$ training examples. 

The two-layer ReLU network is initialized using $a_i \sim \mathcal{N}(0,\frac{2\kappa}{m}),\, w_{i,j} \sim \mathcal{N}(0,2\kappa/d)$. We use $\kappa=1$ and train the regularized models using the Adam optimizer~\cite{kingma2014adam} for $T=10,000$ steps, unless it is specified otherwise. The initial learning rate is set to be $0.001$, and it is then multiplied by a decay factor of $0.1$ at $0.7T$ and again at $0.9T$. We set the trade-off parameter $\lambda=0.1\lambda_n$\footnote{Our proof of theoretical results require $\lambda \geq \lambda_n$. However, this condition is not necessarily optimal.} .

\subsection{Shaper bounds for the generalization gap}
Theorem~\ref{thm: posterior-gen-gap} shows that the generalization gap is bounded by $\frac{\|\theta\|_{\cP}}{\sqrt{n}}$ up to some logarithmic terms. Previous works~\cite{neyshabur2018towards,dziugaite2017computing} showed that (stochastic) gradient descent tends to find  solutions with huge norms, causing the a posterior bound to be vacuous. In contrast, our theory suggests there exist good solutions (i.e. solutions with small generalization error) with small norms, and these solutions can be found by the explicit regularization.

To see how this works in practice, we trained both the regularized models and un-regularized models ($\lambda=0$) for fixed network width $m=$10,000. To cover the over-parametrized regime, we also consider the case $n=100$ where $m/n=100 \gg 1$. The results are summarized in Table~\ref{tab: gen-bound}. 

\begin{table}[!h]
\renewcommand{\arraystretch}{1.1}
\centering
\begin{tabular}{c|c|c|c|c|c}
\hline 
dataset &  $\lambda$      &n & training accuracy& testing accuracy  &  $\frac{\|\theta\|_{\cP}}{\sqrt{n}}$ \\\hline \hline
\multirow{4}{*}{CIFAR-10} & \multirow{2}{*}{$0$} & $10^4$ & 100\% & $84.5\%$ & 58 \\ 
                          &                          & 100   & $100\%$ & $70.5\%$ & $507$ \\\cline{2-6}
                          & \multirow{2}{*}{$0.1$} & $10^4$ &  $87.4\%$ & $86.9\%$ & $\textbf{0.14}$ \\
                          &                          &100 & $91.0\%$ & $72.0\%$ & $\textbf{0.43}$\\
\hline
\multirow{4}{*}{MNIST}    & \multirow{2}{*}{$0$} & $6\times 10^4$ & $100\%$ & $98.8\%$ & 58 \\
                          &                          & 100   & $100\%$ &  $78.7\%$ & 162 \\\cline{2-6}
                          & \multirow{2}{*}{$0.1$} & $6\times 10^4$ &  $98.1\%$ & $97.8\%$ & $\textbf{0.27}$ \\
                          &                             & 100 & $100\%$ &  $74.9\%$ & \textbf{0.41} \\
\hline  
\end{tabular}
\caption{Comparison of regularized $(\lambda=0.1)$ and un-regularized $(\lambda=0)$ models.  
For each case, the experiments are repeated for $5$ times, and the mean values are reported.}
\label{tab: gen-bound}
\end{table}

As we can see, the test accuracies of  the regularized and un-regularized solutions are generally comparable, but the values of   $\frac{\|\theta\|_{\cP}}{\sqrt{n}}$, which serve as an upper bound for the generalization gap,  are drastically different. The bounds 
for the un-regularized models are always vacuous, as was observed in \cite{dziugaite2017computing,neyshabur2018towards,arora2018stronger}.  
In contrast, the bounds for the regularized models are always several orders of magnitude smaller than that for the un-regularized models.  This is consistent with the theoretical prediction in Proposition~\ref{pro: reg-estimator}.

To further explore the impact of over-parametrization, we trained various models with different widths. For both datasets, all the training examples are used. In Figure~\ref{fig: pathnorm-width}, we display how the value of $\frac{\|\theta\|_{\cP}}{\sqrt{n}}$ of the learned solution varies with the network width. We find that  for the un-regularized model this quantity increases with network width, whereas for the regularized model it is almost constant. This is consistent with our theoretical result.

\begin{figure}[!h]
\centering
\includegraphics[width=0.4\textwidth]{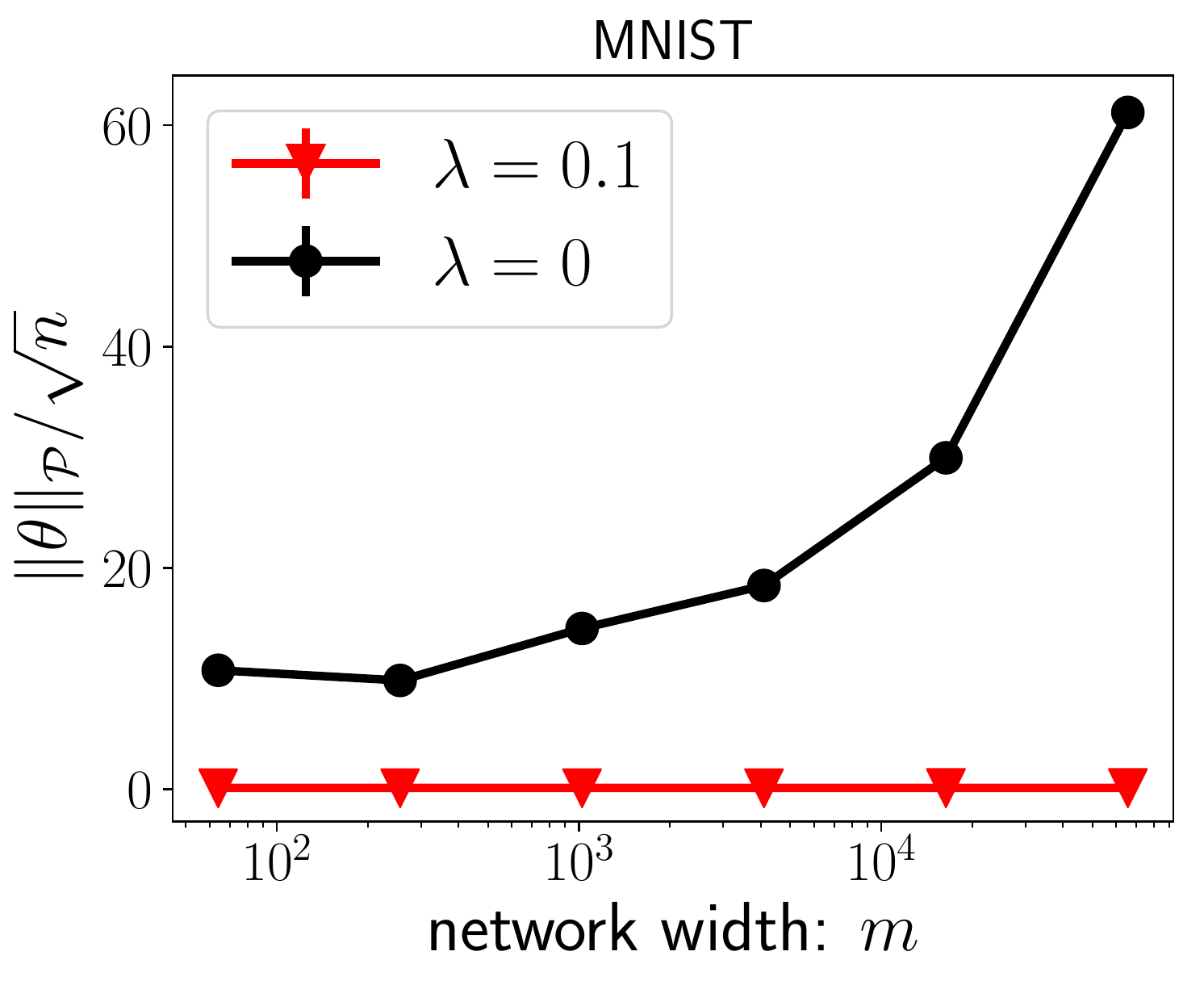} 
\hspace*{5mm}
\includegraphics[width=0.4\textwidth]{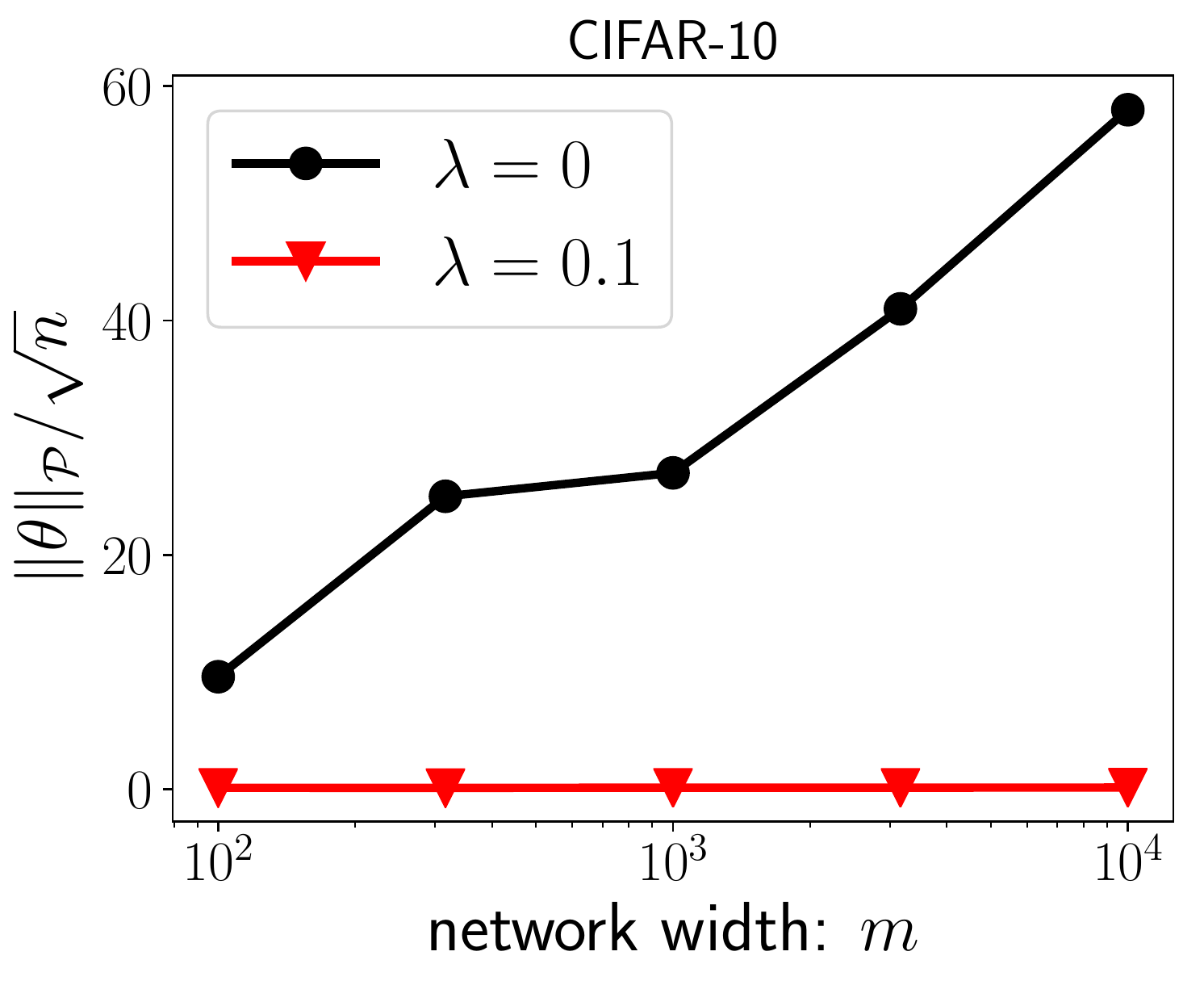}

\caption{ Comparison of the path norms between the regularized and un-regularized solutions for varying widths.}
\label{fig: pathnorm-width}
\end{figure}

\subsection{Dependence on the Initialization}
Since the neural network model  is non-convex, it is interesting to see how  initialization affects the performance of the 
different models, regularized and un-regularized, especially in the over-parametrized regime. To this end, we fix $m=10000, n=100$ and vary  the variance of random initialization $\kappa$. The  results are reported in Figure~\ref{fig: init-testacc}. In general, we find that regularized models are much more stable than the un-regularized models. For large initialization, the regularized model always performs significantly better. 

\begin{figure}[!h]
\centering
\includegraphics[width=0.4\textwidth]{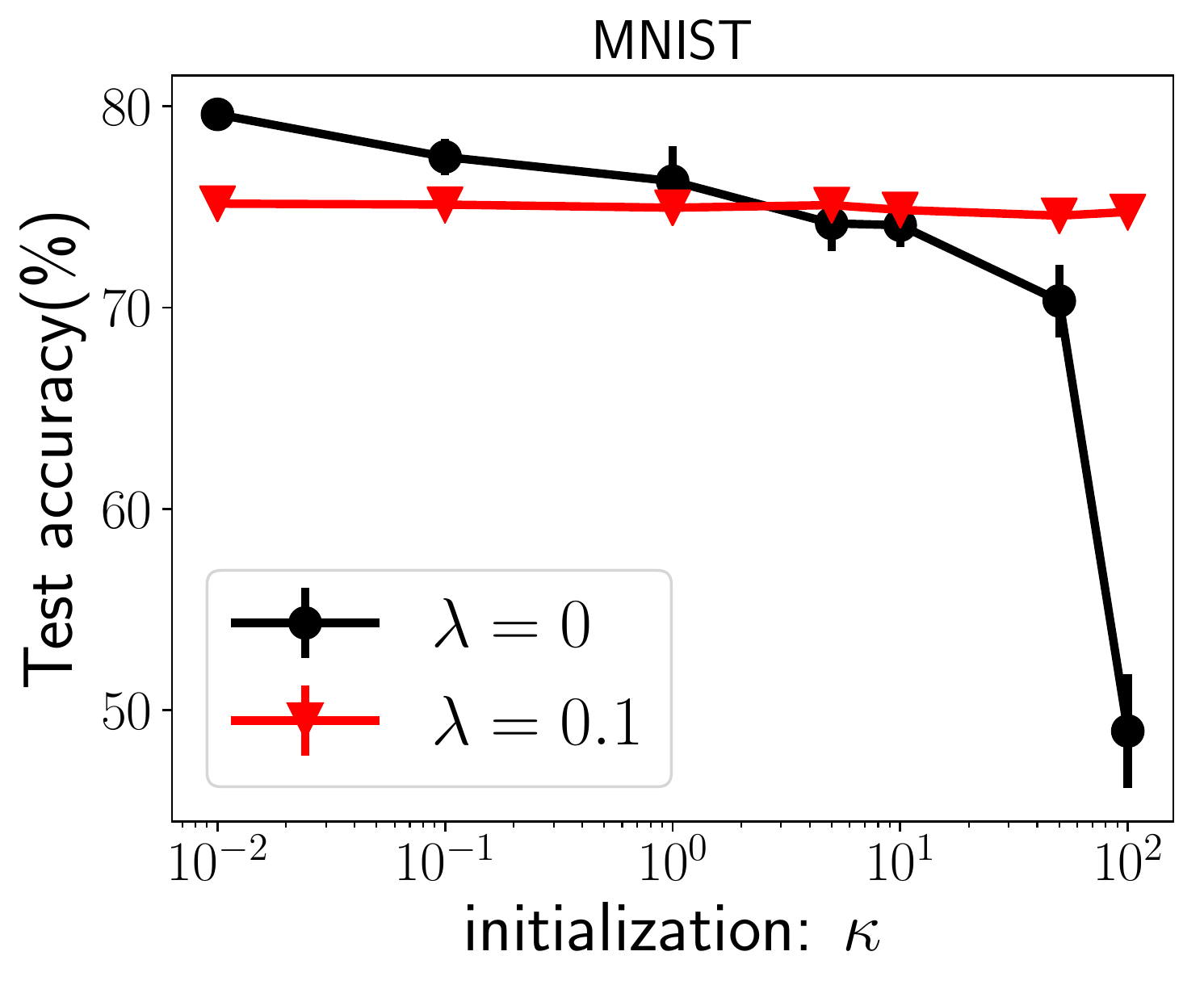} 
\hspace*{5mm}
\includegraphics[width=0.4\textwidth]{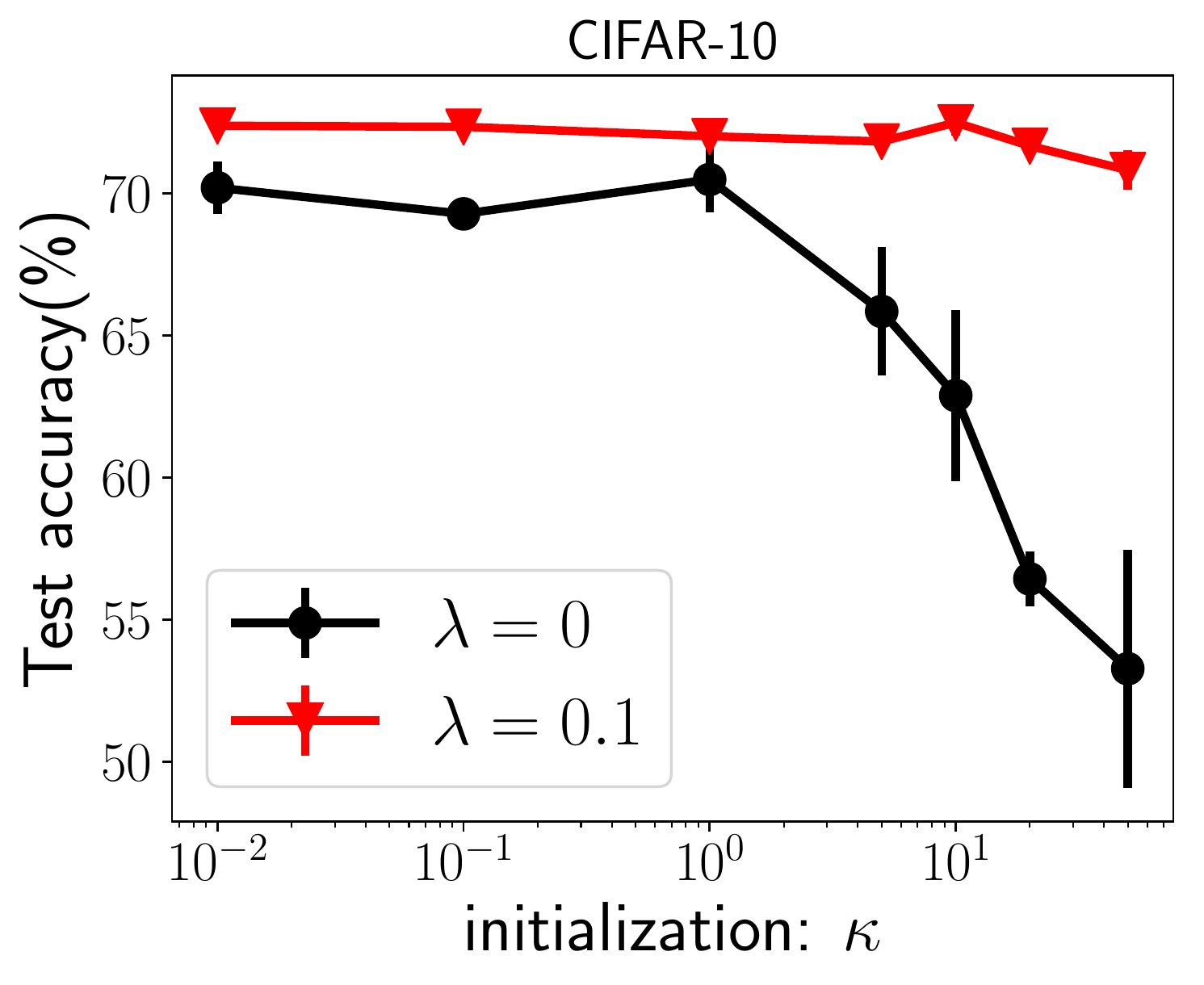}
\caption{Test accuracies of solutions obtained from different initializations. Each experiment is repeated for $5$ times, and we report the mean and standard deviation.}
\label{fig: init-testacc}
\end{figure}

\section{Conclusion}
\label{sec: discuss}
In this paper, we proved nearly optimal a priori estimates of the population risk for learning two-layer neural networks.  
Our results also give some insight regarding the advantage of neural network models over the kernel method.
We should also mention that the main result of this paper has also been extended to deep residual network models in \cite{ma2019priori}.

The most unsatisfactory aspect of our result is that it is proved for the regularized model
since practitioners rely on the so-called implicit regularization.  
At the moment it is unclear where the ``implicit regularization'' comes from and how it actually works. 
Existing works consider special initialization schemes and require strong assumptions on the target function~\cite{brutzkus2018sgd,allen2018learning,daniely2017sgd,leiwu2019two,ma2019analysis}. 
In particular, the work in \cite{leiwu2019two,ma2019analysis} demonstrates clearly that in the regimes considered the neural network models
are no better than the kernel method in terms of implicit regularization.  This is quite unsatisfactory.

There are overwhelming evidence that by tuning the
optimization procedure, including the algorithm, the initialization, the hyper-parameters, etc., one can find 
solutions with superior performance on the test data.
The problem is that excessive tuning and serious experience is required  to find good solutions.
Until we have a good understanding about the mysteries surrounding implicit regularization, the business
of parameter tuning for un-regularized models will remain an art. 
In contrast, the regularized model proposed here is rather robust and much more fool-proof.
Borrowing the terminology from mathematical physics, one is tempted to say that  the regularized model considered here
is ``well-posed'' whereas the un-regularized model is ``ill-posed'' \cite{andre1977solutions}.

\appendix
\section{Proof of Theorem~\ref{pro: approximation}}
\label{sec: appendix-approx}
Without loss of generality, let $(a,\pi)$ be the best representation of $f$, i.e. 
$\gamma_2^2(f)=\EE_{\pi}[|a(w)|^2]$.
Let $U=\{w_j\}_{j=1}^m$ be {\it i.i.d.} random variables sampled from $\pi(\cdot)$, and define 
\[
    \hat{f}_U(x) = \frac{1}{m}\sum_{j=1}^m a(w_j) \sigma(\langle w_j,x\rangle).
\]
Let $L_U=\EE_{x}|\hat{f}_U(x)-f(x)|^2$ denote the population risk,  we have
\begin{align*}
    \EE_{U}[L_U] &= \EE_{x}\EE_{U}|\hat{f}_U(x)-f(x)|^2 \\
    &= \frac{1}{m^2}\EE_{x}\sum_{j,l=1}^m \EE_{w_j,w_l}[(a(w_j)\sigma(\langle w_j,x \rangle)-f(x))(a(w_l)\sigma(\langle w_l,x \rangle)-f(x))]\\
    &\leq \frac{\gamma_2^2(f)}{m}.
\end{align*}
On the other hand, denote  the path norm of $\hat{f}_U(x)$ by $A_U$, we have $\EE_{U}[A_U]=\gamma_1(f)\leq \gamma_2(f)$. 

Define the event $E_1=\{L_U < \frac{3 \gamma_2^2(f)}{m}\}$, and $E_2=\{A_U < 2\gamma_1(f)\}$.
By Markov's inequality, we have 
\begin{align*}
    \PP\{E_1\} &=1 -\PP\{L_U \geq\frac{3 \gamma_2^2(f)}{m}\} \geq  1 - \frac{\EE_U[L(U)]}{3\gamma^2_2(f)/m}\geq \frac{2}{3} \\
    \PP\{E_2\}&= 1- \PP\{A_U \geq 2 \gamma_2(f)\} \geq  1- \frac{\EE[A_U]}{2\gamma_2(f)}\geq \frac{1}{2}.
\end{align*}
Therefore, we have the probability of two events happens together,
\begin{align*}
\PP\{E_1\cap E_2\} = \PP\{E_1\}+\PP\{E_2\} -1 \geq \frac{2}{3}+ \frac{1}{2}-1>0.
\end{align*} 
This completes the proof.

\section{Proof of Theorem~\ref{thm: posterior-gen-gap}}
\label{sec: appendix-gen-bound}
Before we provide the upper bound for the Rademacher complexity of two-layer networks, we first need the following two lemmas.

\begin{lemma}[Lemma 26.11 of \cite{shalev2014understanding}]\label{lemma: l1-linear-class}
 Let $S=(\bx_1,\dots,\bx_n)$ be $n$ vectors in $\RR^d$. Then the Rademacher complexity of $\cH_1 = \{\bx\mapsto \bu\cdot\bx \ |\ \|\bu\|_1\leq 1 \}$ has the following upper bound,
\[
    \hat{\cR}_n(\cH_1) \leq \max_{i}\|\bx_i\|_{\infty} \sqrt{\frac{2\ln(2d)}{n}}
\]
\end{lemma}
The above lemma characterizes the Rademacher complexity of a linear predictor with $\ell_1$ norm bounded by $1$. To handle the influence of nonlinear activation function, we need the following contraction lemma. 
\begin{lemma}[Lemma 26.9 of \cite{shalev2014understanding}]\label{lemma: contraction-rademacher-complexity}
Let $\phi_i: \RR\mapsto\RR$ be a $\rho-$Lipschitz function, i.e. for all $\alpha,\beta\in\RR$ we have $|\phi_i(\alpha)-\phi_i(\beta)|\leq \rho|\alpha-\beta|$. For any $\bm{a}\in\RR^n$, let $\bm{\phi}(\bm{a})=(\phi_1(a_1),\dots,\phi_n(a_n))$, then we have
\[
    \hat{\cR}_n(\bm{\phi}\circ\cH) \leq \rho \hat{\cR}_n(\cH)
\]
\end{lemma}
We are now ready to  estimate the Rademacher complexity of two-layer networks. 

\begin{lemma}\label{lemma: Rademacher}
Let $\cF_Q = \{ f_m(x;\theta)\,|\, \|\theta\|_{\cP}\leq Q\}$ be the set of two-layer networks with path norm bounded by $Q$, then we have
\[
    \hat{\cR}_n(\cF_Q) \leq 2Q \sqrt{\frac{2\ln(2d)}{n}}
\]
\end{lemma}

\begin{proof}
To simplify the proof, we let $c_k=0$, otherwise we can define $\bb_k = (\bb_k^T,c_k)^T$ and $\bx = (\bx^T,1)^T$. 
\begin{align*}
n \hat{\cR}_n(\cF_Q) &= \EE_{\xi} \big[\sup_{\|\theta\|_{\cP}\leq Q} \sum_{i=1}^n \xi_i \sum_{k=1}^m a_k \|\bb_k\|_1 \sigma(\hat{\bb}_k^T\bx_i) \big]\\
& \leq  \EE_{\xi}\big[\sup_{\|\theta\|_{\cP}\leq Q, \|\bu_k\|_1=1} \sum_{i=1}^n \xi_i \sum_{k=1}^m a_k \|\bb_k\|_1 \sigma(\bu_k^T\bx_i)\big] \\
&= \EE_{\xi}\big[\sup_{\|\theta\|_{\cP}\leq Q, \|\bu_k\|_1=1} \sum_{k=1}^m a_k \|\bb_k\|_1 \sum_{i=1}^n \xi_i \sigma(\bu_k^T\bx_i)\big] \\
&\leq \EE_{\xi}\big[\sup_{\|\theta\|_{\cP}\leq Q} \sum_{k=1}^m |a_k \|\bb_k\|_1| \sup_{\|\bu\|_1=1}|\sum_{i=1}^n \xi_i \sigma(\bu^T\bx_i)|\big]\\
&\leq Q\EE_{\xi}\big[\sup_{\|\bu\|_1=1}|\sum_{i=1}^n \xi_i \sigma(\bu^T\bx_i)|\big] \\
&\leq Q\EE_{\xi}\big[\sup_{\|\bu\|_1\leq 1}|\sum_{i=1}^n \xi_i \sigma(\bu^T\bx_i)|\big] 
\end{align*}
Due to the symmetry, we have that 
\begin{align*}
\EE_{\xi}\big[\sup_{\|\bu\|_1\leq 1}|\sum_{i=1}^n \xi_i \sigma(\bu^T\bx_i)|\big] &\leq   \EE_{\xi}\big[\sup_{\|\bu\|_1\leq 1}\sum_{i=1}^n \xi_i \sigma(\bu^T\bx_i) + \sup_{\|\bu\|_1\leq 1} \sum_{i=1}^n - \xi_i \sigma(\bu^T\bx_i)\big]\\
&= 2\EE_{\xi}\big[\sup_{\|\bu\|_1\leq 1}\sum_{i=1}^n \xi_i \sigma(\bu^T\bx_i)\big]
\end{align*}
Since $\sigma$ is  Lipschitz continuous with Lipschitz constant  $1$, by applying Lemma~\ref{lemma: contraction-rademacher-complexity} and Lemma~\ref{lemma: l1-linear-class}, we obtain
\[
    \hat{\cR}_n(\cF_Q) \leq 2Q \sqrt{\frac{2\ln(2d)}{n}}.
\]
\end{proof}

\begin{proposition}\label{pro: fixed-hypothesis}
Assume the loss function $\ell(\cdot,y)$ is $A-$Lipschitz continuous and bounded by $B$, then with probability at least $1-\delta$ we have,
\begin{equation}\label{eqn: path-norm-gen-bound}
    \sup_{\|\theta\|_{\cP}\leq Q} |L(\theta)-\hat{L}_n(\theta)| \leq 4 A Q\sqrt{\frac{2\ln(2d)}{n}} +  B \sqrt{\frac{2\ln(2/\delta)}{n}}
\end{equation}
\end{proposition}
\begin{proof}
Define $\cH_Q = \left\{\ell\circ f \,|\, f\in \cF_Q\right\}$, then we have
    $
    \hat{\cR}_n(\cH_Q) \leq 2BQ\sqrt{\frac{2\ln(2d)}{n}},
    $
which follows from  Lemma~\ref{lemma: contraction-rademacher-complexity} and \ref{lemma: Rademacher}. Then directly applying Theorem~\ref{thm: gen-err-rademacher-complexity} yields the result.
\end{proof}
\begin{proof}({\bf Proof of Theorem~\ref{thm: posterior-gen-gap}}) 
Consider the decomposition  $\cF = \cup_{l=1}^{\infty} \cF_l $, where
 $\cF_l = \{f_m(\bx;\theta)\, |\, \|\theta\|_{\cP}\leq l\}$.
 Let $\delta_l = \frac{\delta}{c\, l^2}$ where $c = \sum_{l=1}^{\infty} \frac{1}{l^2}$. According to Theorem~\ref{pro: fixed-hypothesis}, if we fix $l$ in advance, then with probability at least $1-\delta_l$ over the choice of $S$, we have
\begin{equation}
    \sup_{\|\theta\|_{\cP}\leq l} |L(\theta)-\hat{L}_n(\theta)| \leq 4 A l \sqrt{\frac{2\ln(2d)}{n}} +  B \sqrt{\frac{2\ln(2/\delta_l)}{n}}.
\label{eqn: fixed-l}
\end{equation}
So the probability that there exists at least one $l$ such that \eqref{eqn: fixed-l} fails is at most $\sum_{l=1}^{\infty} \delta_l = \delta$. In other words, with probability at least $1-\delta$, the  inequality~\eqref{eqn: fixed-l} holds for all $l$.

Given an arbitrary set of parameters $\theta$, denote $l_0 = \min\{ l \,|\,\|\theta\|_{\cP}\leq l\}$, 
then  $l_0 \leq \|\theta\|_{\cP}+1$. Equation~\eqref{eqn: fixed-l} implies that
\begin{align*}
|L(\theta)-\hat{L}_n(\theta)| &\leq 4 A l_0 \sqrt{\frac{2\ln(2d)}{n}} +  B \sqrt{\frac{2\ln(2cl_0^2/\delta)}{n}} \\
&\leq 4 A (\|\theta\|_{\cP}+1) \sqrt{\frac{2\ln(2d)}{n}} + B  \sqrt{\frac{2\ln(2c(1+\|\theta\|_{\cP})^2/\delta)}{n}}.
\end{align*}
\end{proof}

\section{Proof of Lemma~\ref{lemma: noise}}
\label{sec: appendix-truncated-risk}
\begin{proof}
Let $Z=f(x;\theta)-f^*(x)-\varepsilon$, then  for any $B\geq 2+\tau_0$, we have
\begin{align*}
|L(\theta) - L_B(\theta)| &= \EE\left[ (Z^2 - B^2)\bm{1}_{|Z|\geq B}  \right] \\
& = \int_{0}^{\infty}\PP\{Z^2-B^2\geq t^2\} d t^2 \leq \int_0^{\infty} \PP\{|Z|\geq \sqrt{B^2+t^2}\}d t^2\\
&\leq \int_{0}^{\infty}\PP\{|\varepsilon|\geq \sqrt{B^2+t^2}-2\} dt^2\\
&= c_0 \int_{B}^{\infty} e^{-\frac{s^2}{2\sigma^2}} ds^2 
= 2c_0\sigma^2  e^{-B^2/2\sigma^2}
\end{align*}
Since $B_n\geq \sigma^2 \ln n$, we have $2c_0 \sigma^2  e^{-\frac{B^2_n}{2\sigma^2}}\leq 2c_0\sigma^2 n^{-1/2}$. We thus complete the proof.
\end{proof}

\vspace*{3mm}

{\bf Acknowledgement:}
The work presented here is supported in part by a gift to Princeton University from iFlytek
and the ONR grant N00014-13-1-0338.

\bibliographystyle{plain}
\bibliography{deeplearning}

\end{document}